\documentclass{eptcs}
\providecommand{\event}{TARK 2019} 
\usepackage{breakurl}             
\usepackage{underscore}           

\usepackage{tikz}
\usepackage{amsmath,amsfonts,amsxtra, amssymb,stmaryrd}
\usepackage{ntheorem}
\usepackage{verbatim}

\makeatletter
\providecommand*{\Dashv}{%
  \mathrel{%
    \mathpalette\@Dashv\vDash
  }%
}
\newcommand*{\@Dashv}[2]{%
  \reflectbox{$\m@th#1#2$}%
}
\makeatother
\usetikzlibrary{positioning,arrows,calc}
\tikzset{modal/.style={>=stealth',shorten >=1pt,shorten <=1pt,auto,node distance=1.5cm,semithick},world/.style={circle,draw,minimum size=0.5cm,fill=gray!0},point/.style={circle,draw,inner sep=0.5mm,fill=black},reflexive above/.style={->,loop,looseness=7,in=120,out=60},reflexive below/.style={->,loop,looseness=7,in=240,out=300},reflexive left/.style={->,loop,looseness=7,in=150,out=210},
  texto/.style = {
    font = \scriptsize,
    align = center,},
reflexive right/.style={->,loop,looseness=7,in=30,out=330}}

\usepackage[utf8]{inputenc}

\newcommand{\norm}[1]{\left\lVert#1\right\rVert}

\newcommand{\nrcr}{\not\kern-0.07cm\mid\!\sim}

\newcommand{\npr}{\not\kern-0.07cm\mid\!\!\!-}

\newtheorem{theorem}{Theorem}
\newtheorem{definition}{Definition}
\newtheorem{Definition}{Definition}
\newtheorem{proposition}{Proposition}

\newtheorem{lemma}{Lemma}
\newtheorem{Lemma}{Lemma}
\newtheorem{Example}{Example}
\newtheorem*{Example-non}{Example}
\newtheorem{corollary}{Corollary}

\newenvironment{proof}{\noindent{\bf Proof.}}{\hfill $\blacksquare$}

\providecommand{\urlalt}[2]{\href{#1}{#2}}
\providecommand{\doi}[1]{doi:\urlalt{http://dx.doi.org/#1}{#1}}

\title{Learning Probabilities: Towards a Logic of Statistical Learning}
\author{Alexandru Baltag
\institute{ILLC, University of Amsterdam\\
Amsterdam, Netherlands}
\email{A.Baltag@uva.nl}
\and
Soroush Rafiee Rad \qquad\qquad\qquad Sonja Smets
\institute{\qquad\qquad University of Bayreuth \qquad\qquad ILLC, University of Amsterdam\\
\qquad\qquad\qquad Bayreuth, Germany \qquad\qquad\qquad\qquad Amsterdam, Netherlands}
\email{\quad soroush.r.rad@gmail.com \quad\qquad S.J.L.Smets@uva.nl}
}

\def\titlerunning{Learning Probabilities}
\def\authorrunning{A. Baltag et al.}

\begin{document}
\maketitle

\begin{abstract}
We propose a new model for forming beliefs and learning about unknown probabilities (such as the probability of picking a red marble from a bag with an unknown distribution of coloured marbles). The most widespread model for such situations of `radical uncertainty' is in terms of imprecise probabilities, i.e. representing the agent's knowledge as a set of probability measures. We add to this model a plausibility map, associating to each measure a plausibility number, as a way to go beyond what is known with certainty and represent the agent's beliefs about probability. There are a number of standard examples: Shannon Entropy, Centre of Mass etc. We then consider learning of two types of information: (1) learning by repeated sampling from the unknown distribution (e.g. picking marbles from the bag); and (2) learning higher-order information about the distribution (in the shape of linear inequalities, e.g. we are told there are more red marbles than green marbles). The first changes only the plausibility map (via a `plausibilistic' version of Bayes' Rule), but leaves the given set of measures unchanged; the second shrinks the set of measures, without changing their plausibility. Beliefs are defined as in Belief Revision Theory, in terms of truth in the most plausible worlds. But our belief change does not comply with standard AGM axioms, since the revision induced by (1) is of a non-AGM type. This is essential, as it allows our agents to learn the true probability: we prove that the beliefs obtained by repeated sampling converge almost surely to the correct belief (in the true probability). We end by sketching the contours of a dynamic doxastic logic for statistical learning.\end{abstract}

\section{Introduction}

Our goal in this paper is to propose a new model for {\it learning a probabilistic distribution},
in cases that are commonly characterized as those of ``radical uncertainty" \cite{wally} or ``Knightian uncertainty'' \cite{CV}. As an example, consider an urn full of marbles, coloured red, green and black, but with an unknown distribution. What is then the probability of drawing a red marble? In such cases, when the agent's information is not enough to determine the probability distribution, she is typically left with a huge (usually infinite) \emph{set} of probability assignments. If she never goes beyond what she knows, then her only `rational' answer should be ``I don't know": she in a state of \emph{ambiguity}, and she should simply consider possible \emph{all} distributions that are consistent with her background knowledge and observed evidence. Such a ``Buridan's ass" type of rationality will not help our agent much in her decision problems.

Our model allows the agent to go beyond what she knows with certainty, by forming \emph{rational qualitative beliefs about} the unknown distribution, beliefs based on the inherent plausibility of each possible distribution. For this, we assume the agent is endowed with an initial \emph{plausibility map}, assigning real numbers to the possible distributions. To form beliefs, the agent uses an AGM-type of \emph{plausibility maximization}: she believes the most plausible distribution(s). So we equate ``belief" with \emph{truth in all the most plausible worlds}, whenever such most plausible worlds exist; while in more general settings, we follow the standard generalization of this notion of belief as ``truth in all the worlds that are \emph{plausible enough}".  This is the standard definition of qualitative belief in Logic and Belief Revision Theory. As a consequence, all the usual KD45 axioms of doxastic logic will be valid in our framework. The plausibility map encodes the agent's background beliefs and a priori assumptions about the world. For instance, an agent whose a priori assumptions include the Principle of Indifference will use Shannon entropy as her plausibility function, thus initially believing the most non-informative distribution(s). An agent who assumes some form of Ockham's Razor will use as plausibility some measure of simplicity, thus her initial belief will focus on the simplest distribution(s), etc. Note that, although our plausibility map assigns real values to probability distributions, this account is essentially different from the ones using so-called ``second-order probabilities"(i.e. probabilities distributions defined on the set of probability distributions). Plausibility values are only relevant in so far as they induce a qualitative order on distributions. In contrast to probability, plausibility is \emph{not cumulative} (in the sense that the low-plausibility alternatives do not add up to form more plausible sets of alternatives), and as a result only the distributions with the \emph{highest} plausibility play a role in defining beliefs.

Our model is not just a way to ``rationally" select a Bayesian prior, but it also comes with a rational method for \emph{revising beliefs} in the face of new evidence. In fact, it can deal with \emph{two types of new information}: first-order evidence gathered by repeated \emph{sampling} from the (unknown) distribution; and higher-order information about the distribution itself,  coming in the form of \emph{linear inequality constraints} on that distribution. To see the difference between the two types of new evidence, take for instance the example of a coin. As it is well known any finite sequence of Heads and Tails is consistent with all possible biases of the coin. As such, any number of finite repeated samples {\em will not} shrink the set of possible biases, though they may make increase the plausibility of some biases. Thus this type of information changes only the plausibility map but leaves the given set of measures unchanged. The second type of information, on the other hand, shrinks the set of measures, without changing their plausibility. As for instance learning that the coin has a bias towards Tail (e.g. by weighing the coin, or receiving a communication in this sense from the coin's manufacturer) eliminates all distributions that assign a higher probability to Heads. It is important to notice, however, that even with higher order information it is hardly ever the case that the distribution under consideration is fully specified. In our coin example, a known bias towards Tails will still leave a infinite set of possible biases consistent. Even a good measurement by weighting will leave open a whole interval of possible biases. In this sense a combination of observations and higher order information will \emph{not} in general allow the agent to come to {\em know} the correct distribution in the standards sense in which the term knowledge is used in doxastic and epistemic logics. Instead, it may eventually allow her to come to \emph{believe} the true probability (at least, with a high degree of accuracy). This ``convergence in belief" is what we aim to capture in this paper.

Our belief revision mechanism after sampling is non-Bayesian (and also different from the AGM belief revision), though it incorporates a ``plausibilistic" version of Bayes' Rule. Instead of updating her prior belief according to this rule (and disregarding all other possible distributions), the agent keeps all possibilities in store and \emph{revises instead her plausibility relation} using an analogue of Bayes' Rule. After that, her new belief will be formed in a similar way to her initial belief: by maximizing her (new) plausibility. The outcome is different from simply performing a Bayesian update on the `prior': qualitative jumps are possible, leading to abandoning ``wrong" conjectures in a non-monotonic way. This results in a \emph{faster} convergence-in-belief to the true probability in \emph{less restrictive conditions} than the usual Savage-style convergence through repeated Bayesian updating.\footnote{In contrast to Savage's theorem, our update ensures convergence even in the case that the initial set of possible distributions is infinite (indeed, even in the case we start with the uncountable set of \emph{all} distributions). Moreover, in the finite case (where Savage's result does apply), our update is guaranteed to converge in finitely many steps, while Savage's theorem only ensures convergence in the limit.} Note also that the belief update induced by sampling does \emph{not} satisfy all the standard AGM axioms for belief revision. This is essential for learning the true probability from repeated sampling: since every sample is logically consistent with every distribution, an AGM learner would never change her initial belief!

The second type of evidence (higher-order information about the distribution) induces a more familiar kind of update: the distributions that do not satisfy the new information (typically given in the former of linear inequalities) are simply eliminated, then beliefs are formed as before by focusing on the most plausible remaining distributions. This form of revision is known as AGM \emph{conditioning} in Belief Revision Theory (and as \emph{update}, or ``public announcement", in Logic), and satisfies all the standard AGM axioms.

The fact that in our setting there are two types of updates should not be so surprising. It is related to the fact that our static framework consists of two different semantic ingredients, capturing two different attitudes: the \emph{set} of possible distributions (encoding the agent's \emph{knowledge} about the correct distribution), and the \emph{plausibility} map (encoding the agent's \emph{beliefs}). The second type of (higher-order) information directly affects the agent's knowledge (by reducing the set of possibilities), and only indirectly her beliefs (by restricting the plausibility map to the new set, so beliefs are only updated with fit the new knowledge). Dually, the first type of (sampling) evidence acts directly affects the agent's beliefs (by changing the plausibility in the view of the sampling results), and only indirectly her knowledge (since e.g. she knows her new beliefs).

The plan of this paper follows. We start by reviewing some basic notions, results and examples on probability distributions (Section 2). Then in Section 3, we define our main setting (probabilistic plausibility frames),
consider a number of standard examples (Shannon Entropy, Center of Mass etc), then formalize the updates induced by the two types of new information, and prove our main result on convergence-in-belief. In Section 4, we sketch the contours of a dynamic doxastic logic for statistical learning. We end with some concluding remarks and a brief comparison with other approaches to the same problem (Section 5).


\section{Preliminaries and Notation}

Take a finite set $O=\{o_1, \ldots, o_n\}$ and  let
$M_{O} = \{\mu\in [0,1]^O \vert \,
\sum_{o \in O} \mu(o)=1\}$
be the set of probability mass functions on $O$, which we identify with the corresponding probability functions on $\mathcal{P}(O)$.
Let $\Omega= O^{\infty} = O^{\mathbb{N}}$ be the set of infinite sequences from $O$, which we shall refer to as {\em observation streams}. For any $\omega \in \Omega$ and $i \in \mathbb{N}$, we write $\omega_i$ for the $i$-th component of $\omega$, and $\omega^{i}$ for its initial segment of length $i$, that is $\omega_1, \ldots, \omega_i$.  For each $o \in O$ we define the sets $o^{j}$ to be the cylinders $\, o^{j}=\{\omega \in \Omega \, \vert \,\, \omega_{j}=o \} \subseteq \Omega.$
Let $\mathcal{A} \subseteq \mathcal{P}(\Omega)$ be the $\sigma$-algebra of subsets of $\Omega$ generated by the cylinders. Every probability distribution $\mu \in M_{O}$ induces a unique probability function, $\hat{\mu}$ over $(\Omega, \mathcal{A})$ by setting $\hat{\mu}(o^{j})= \mu(o)$ which extends to all of $\mathcal{A}$ using independence. Let $\mathcal{E}$ be the subalgebra of $\mathcal{A}$ that is closed under complementation and  {\em finite} unions and intersections of the cylinder sets. Then $\mathcal{E}$ will capture the set of events generated by finite sequences of observations.

\begin{Example}
Let $O=\{H, T\}$ be the possible outcomes of a coin toss. Then $\Omega$ will be streams of {\em Heads} and {\em Tails} representing infinite tosses of the coin, e.g. HTTHHH.... And $H^{j}$ (res. $T^{j}$) will be the set of streams of observations in which the $j$-th toss of the coin has landed Heads (res. Tails). The set $M_O$ will be the set of possible biases of the coin.
\end{Example}

\begin{Example}
Let $O=\{R, B, G\}$ be the possible outcomes for a draw from an urn filled with marbles, coloured Red, blue and Green. Then the set $M_O$ will be the set of different distribution of coloured marbles in the urn, $\Omega$ will be streams of {\em R}, {\em B} and {\em G} representing infinite draws from the urn, and $R^{j}$ (res. $B^{j}$ or $G^{j}$) will be the set of streams of draws in which the $j$-th draw is a Red (res. Blue or Green) marble.
\end{Example}

\par\noindent{\bf Topology on $M_{O}$} Notice that a probability function $\mu \in M_O$, defined over the set $O=\{o_1, \ldots, o_n\}$, can be identified with an $n$-dimensional vector $(\mu(o_1), \ldots, \mu(o_n))$, corresponding to the probabilities assigned to each $o_i$ respectively. Let $\mathcal{D}_O := \{\vec{x} \in [0,1]^{n} \, \vert \, \sum x_i =1\},$ then every $\mu \in M_O$ can be identified with the point $\vec{\mu} \in \mathcal{D}_O \subset [0, 1]^n$. Thus probability functions in $\mathcal{M}_{O}$ live in the space $\mathbb{R}^{n}$ (or more precisely $[0, 1]^{n}$). In the other direction every $\vec{x} \in \mathcal{D}_O$ defines a probability function $x$ on $O$ by setting $x(o_i) = \vec{x}_i$. This gives a one to one correspondence between $M_O$ and $\mathcal{D}_O$.
There are various metric distances that can be defined on the space of probability measures over a (finite) set $O$ many of which are known to induce the same topology. Here we will consider the \emph{standard topology} of $\mathbb{R}^n$,induced by the Euclidean metric: for $\vec{x}, \vec{y} \in \mathbb{R}^n$, put $d(\vec{x}, \vec{y}) := \sqrt{ \sum_{i=1}^{n}(x_i-y_i)^{2}}$; a basis for the standard topology is given by the family of all \emph{open balls ${\mathcal B}_\epsilon(\vec{x})$ centred at some point $\vec{x}\in \mathbb{R}^n$ with radius $\epsilon>0$}; where
$${\mathcal B}_\epsilon(\vec{x})= \{\vec{y} \in \mathbb{R}^n \, \vert\, d(\vec{x},\vec{y})<\epsilon\}.$$

\begin{proposition}\label{comp}
For a finite set $O$, the set of probability mass functions on $O$, $\mathcal{M}_{O}$, is compact in the standard topology.
\end{proposition}
\begin{proof}
Check that the set $\{\vec{X} \in [0, 1]^{n}\, \vert \, \sum_{i=1}^{n} x_{i}=1\}$ is compact in $\mathbb{R}^{n}$.
\end{proof}

We will make use of the following well known facts:

\begin{proposition}\label{compact-prop}
 Let $X, Y$ be compact topological spaces, $Z \subseteq X$ and $f: X \subseteq Y$
\par\noindent
(1) Every closed subset of $X$ is compact.
\par\noindent
(2) If $f$ is continuous, then $f(X)$ is compact.
\par\noindent
(3) If $Z$ is compact then it is closed and bounded.
\end{proposition}
\begin{proof}
See \cite{IntroAnaly}, Theorem 1.40 and Proposition 1.41.
\end{proof}
\begin{proposition}\label{max-attain}
Let $X$ be a compact topological space and $f: X \to \mathbb{R}$ a continuous function on $X$. Then $f$ is bounded and attains its supremum.
\end{proposition}
\begin{proof}
See \cite{IntroAnaly}, Theorem 7.35.
\end{proof}
\begin{theorem}[Hein-Cantor]
Let $M, N$ be two metric spaces and $f: M \to N$ be continuous. If $M$ is compact then $f$ is uniformly continuous.
\end{theorem}
\begin{proof}
See \cite{Rud}.
\end{proof}

\section{Probabilistic Plausibility Frames}

A \emph{probabilistic plausibility frame} over a finite set $O$ is a structure $\mathcal{F}= (M, pla)$ where $M$ is a subset of $M_O$, called the set of ``possible worlds", and $pla: M_{O}  \to [0 , \infty)$ is a continuous function s.t. the derivative $pla'$ is also continuous.

So our possible worlds are just mass functions on $O$. Here are some canonical examples of probabilistic plausibility frames:
\begin{itemize}
\item (a) \emph{Shannon Entropy} as plausibility: Let $pla:M_O \to [0, \infty)$ be given by the Shannon Entropy,
$pla(\mu)= Ent(\mu)= -\sum_{o \in O} \mu(o)\log(\mu(o)).$
\par\noindent
Then $(M_O, Ent)$ is a probabilistic plausibility frame. Here the most plausible distribution will be the one with highest Shannon entropy.
\item (b) \emph{Centre of Mass} as plausibility: Let $pla :M_O \to [0, \infty)$ be given by the Centre of Mass,
$pla(\mu)= CM (\mu)= \sum_{o \in O} \log(\mu(o)).$
\par\noindent
Then $(M_O, CM)$ is a probabilistic plausibility frame. Here the plausibility ranking will be given in terms of typicality, and higher plausibility will be given to those probability functions that are closer to the average of $M_O$.
\end{itemize}

\begin{Example-non} {\bf 1.} {\em (continued)}. In the absence of any information about the coin the set of possible biases will be the set $M_O$ of \emph{all} probability mass functions on $\{H, T\}$. Then $(M_O, Ent)$ is a probabilistic plausibility frame, where the highest plausibility will be given to the distribution with highest entropy: the fair-coin distribution $\mu^{eq}$ (since for every $\nu \neq \mu^{eq}$ we have $Ent(\nu) < Ent(\mu^{eq})$).
\end{Example-non}

One of the main motivations for developing the setting that we investigate here is to capture the {\em learning process} as iterated revision that results from receiving new information. As was pointed out earlier one type of information essentially trims the space of possible probability measures by deleting certain candidates. There is however, a softer notion of revision, imposed by observations, that does not eliminate any candidate but rather changes the plausibility ordering over them. With this in mind, the next question we need to clarify is how the plausibility order is to be revised in light of new observations.

\begin{definition}[\emph{Conditionalisation}]\label{def2} Let $pla:M_{O} \to [0, \infty)$, and define $pla( . \vert  .):   \mathcal{E} \times M_{O} \to [0, \infty)$, by $\,\, pla( \mu \vert e):= pla(\mu) \hat{\mu}(e)$. When $e \in \mathcal{E}$ is fixed, this yields a \emph{conditional probability function} $pla_{e}: M_{O} \to [0, \infty)$ given by $pla_{e}(\mu):= pla(\mu \, \vert \, e)$.
\end{definition}

Conditioning plausibilities is clearly a higher-level, ``plausubilistic" analogue of Bayesian conditions. It allows us to update the relative rankings of probability distributions, and thus captures a notion of \emph{learning through sampling}. The next three results in Lemma \ref{cont}, and Propositions \ref{prop6} and \ref{cont-c} ensure that the conditionalisation of the plausibility function given by Definition \ref{def2} behaves correctly. In particular, Lemma \ref{cont} and Corollary \ref{cont-c} show that the properties of a plausibility function in our frames is preserved by the conditionalisation and Proposition \ref{prop6} guarantees that the result of repeated conditionalisation is independent of the order. This is important as it ensures that what the agents come to believe is the result of what they learn and not the order in which they learn them.
\begin{lemma}\label{cont}
For each $e \in \mathcal{E}$, the mapping $F_e: M_O \to [0, 1]$ defined as $F_e(\mu) := \hat{\mu}(e)$, is continuous with respect to $\mu$.
\end{lemma}
\begin{proposition}\label{cont-c}
If $pla$ is a plausibility function on $M_O$ and $e \in \mathcal{E}$, then $pla_{e}$ is a plausibility function.
\end{proposition}
\begin{proof}
Follows from the definition using Lemma \ref{cont}.
\end{proof}
\begin{proposition}\label{prop6}
For $\mathcal{M}_{O}$ as above and $pla: \mathcal{M}_{O} \to [0, \infty)$ and $e, e' \in \mathcal{E}$: $\,\, (pla_{e})_{e'}= pla_{e \cap e'}.$
\end{proposition}
\begin{proof}
See appendix
\end{proof}

\begin{Example-non}
{\bf 1.} {\em (continued)} Take  the frame $(M_O, Ent)$ as before where $M_O$ is the set of all biases of the coin and $Ent$ is the Shannon Entropy. Remember that $\mu^{eq}$ is the unique maximiser of $Ent$ on $M_O$. Let $e \in \mathcal{E}$, be the event that ``the first three tosses of the coin have landed on Heads''. After observing $e$, the new plausibility function is given by $\,\, pla_{e}(\mu)= pla(\mu) \hat{\mu}(e)=Ent(\mu) \hat{\mu}(e).$
\par\noindent
Thus the most plausible probability function will no more be $\mu^{eq}$ and one with a bias towards Heads will become more plausible. Let $\mu_1, \mu_2$ and $\mu_3$ be such that $\mu_1(Heads)= 3/4$, $\mu_2(Heads)=0.8$ and $\mu_3(Heads)= 0.9$ then it is easy to check that $\,\, pla_{e}(\mu_1) < pla_{e}(\mu_2)> pla_{e}(\mu_3).$
\end{Example-non}
Our rule for updating plausibility relation weights the plausibility of each world with how much it respects the obtained evidence. In this way worlds that better correspond to the evidence are promoted in plausibility.
\begin{proposition}\label{cons}
Let $M \subseteq M_O$ be closed. Then for all $e \in \mathcal{E}$, there exists some $\mu \in M$ with highest plausible (i.e. s.t. $pla_{e}(\mu) \geq pla_{e}(\mu')$ for all $\mu' \in M$).
\end{proposition}

\begin{proof}
Using Lemma \ref{cont}, the result follows as corollary of Proposition \ref{max-attain}.
\end{proof}

\begin{definition}[Knowledge and Belief]
Let $P \subseteq M$ be a ``proposition" (set of worlds) in a frame $(M,pla)$. We say that $P$ is {\em known}, and write $K(P)$, if all $M$-worlds are in $P$; i.e. $M\subseteq P$.
We say that $P$ is {\em believed} in frame ${\mathcal F}=(M, pla)$, and write $B(P)$, if and only if all ``plausible enough" $M$-worlds are in $P$; i.e. $\{\nu\in M \, \vert \,  pla(\nu) \geq pla(\mu)\}\subseteq P$ for some $\mu\in M$.
\end{definition}

As mentioned in the Introduction, this is the standard notion of belief in Logic and Belief Revision Theory, see \cite{SA, SA2, Jo} for more justification of this definition.

\begin{definition}[Two Forms of Conditionalisation] Let $P \subseteq M_O$ be a ``proposition" (set of distributions).
For an event $e \in \mathcal{E}$, we say that $P$ is {\em believed conditional on $e$} in frame $(M, pla)$, and write $B (P \vert e)$, if and only if all $M$-worlds that are ``plausible enough given $e$" are in $P$; i.e. $\{\nu\in M\, \vert \,  pla_e(\nu) \geq_e pla(\mu)\}\subseteq P$ for some $\mu\in M$.
For a proposition $Q\subseteq M$, we say that  $P$ is {\em believed conditional on $Q$} in frame $(M, pla)$, and write $B (P \vert Q)$, if and only if all plausible enough $Q$-worlds are in $P$; i.e. $\{\nu\in Q \, \vert\, pla(\nu) \geq pla(\mu)\}\subseteq P$ for some $\mu\in Q$.
\end{definition}
It should be clear that $B(P)$ is equivalent to $B(P \vert \Omega)$ and to $B(P \vert M)$, where the set $\Omega$ of all observation streams represents the \emph{tautological event} (corresponding to ``no observation") and the set of $M$ of all worlds represents the \emph{tautological proposition} (corresponding to ``no further higher-order information").

Belief is always consistent, and in fact it satisfies all the standard $KD45$ axioms of doxastic logic.
Conditional belief is consistent whenever the evidence is (i.e. if $e\not=\emptyset$, then $B(P|e)$ implies
$P\not=\emptyset$, and similarly for $B(P|Q)$).
In fact, when the set of worlds is closed, our definition is equivalent to the standard definition of belief (and conditional belief) as ``truth in all the most plausible worlds":
\begin{proposition}\label{prop5}
If $M \subseteq M_O$ is closed, then  $B (P \vert e)$ holds if $\{\mu\in M \, \vert\, pla_e(\mu)\geq pla_e(\mu') \mbox{ for all } \mu'\in M\}\subseteq P$.
\end{proposition}
\begin{proof}
See appendix.
\end{proof}

We are now in the position to look into the learnability of the correct probability distribution via plausibility-revision induced by repeated sampling.
\begin{theorem}\label{learn}
Take a finite set $O$ of outcomes and consider a frame $\mathcal{M}=(M, pla)$ with $M \subseteq M_O$. Suppose that the correct probability is $\mu\in M$ and that $pla(\mu) \neq 0$. Then, with $\mu$-probability 1, the agent's belief will eventually stay arbitrarily close to the correct probability distribution after enough many observations. More precisely, for every $\epsilon>0$, we have
$$\mu(\{ \omega \in \Omega \, \vert \,\, \exists K\,\, \forall m \geq K \,\,\, \,\, B ( {\mathcal B}_\epsilon(\mu)\,\vert\, \omega^{m}) \mbox{ holds in $M$}\})=1$$
(where recall that ${\mathcal B}_\epsilon(\mu)=
\{\nu \in M \vert d(\mu, \nu) < \epsilon \}$).
\end{theorem}

\begin{proof}
See appendix.
\end{proof}
\begin{corollary}
Suppose that $M \subseteq M_O$ is finite, and the correct probability is $\mu\in M$, with $pla(\mu)\neq 0$. Then, with $\mu$-probability 1, the agent's belief will settle on the correct probability distribution $\mu$ after finitely many observations:
$$\, \mu(\{ \omega \in \Omega \, \vert \,\, \exists K\,\, \,\,\,\, \, for\,\, all\,\, m \geq K  \,\, B(\{\mu\} \, \vert \, \omega^{m}) \mbox{ holds in $M$}\})=1.$$
\end{corollary}

\begin{proof} Apply the previous Theorem to some $\epsilon>0$ small enough so that $\{\nu\, \vert \, d(\mu, \nu) < \epsilon\} \cap M=\{\mu\}$.
\end{proof}
It is important to note the differences between our convergence result and the Savage style convergence results in the Bayesian literature that we mentioned in the Introduction. Savage's theorem only works for a finite set of hypotheses (corresponding to finite or countable $M$), so that the prior can assume a non-zero probability for each. Ours does not need this assumption and indeed, it works on the whole $M_O$, since we don't put a probability over hypothesis (probability measures), but rather a plausibility. Also, in the case of a finite set of hypotheses/distributions, our approach converges in finitely many steps (while Savage's still converges only in the limit).

\section{Towards a Logic of Statistical Learning}

In this section we will develop the logical setting that can capture the dynamics of learning described above. As was originally intended our logical language will be designed as to accommodate both type of information, i.e. finite observations and higher order information expressed in terms of linear inequalities. As we pointed out at the start there is a fundamental distinction between these two types of information which is reflected in the way that ingredients of our logical language are interpreted. The observations are interpreted in a $\sigma$-algebra $\mathcal{E} \subseteq \mathcal{P}(\Omega)$ and are not themselves formulas in our logical language as they do not correspond to properties over the set of probability measures. The reason, as described before, lies in the fact that no finite sequence of observations can rule out any possible probability distribution and as such do not single out any subset of the domain. The formulas of our logical language will instead be statements concerning the probabilities of observations given in terms of linear inequalities and logical combinations thereof as well as the statements concerning the dynamics arising from such finite observations.

Our set of \emph{propositional variables} is the set of outcomes $O=\{o_1, \ldots, o_n\}$. The set of formulas, in our language, $FL_{LS}$, is inductively defined as
$$\phi \,::= \top\, \vert\, \sum_{i=1}^{m} a_i w(o_{i}) \geq c \, \vert \, \phi \wedge \phi \, \vert  \, \vert \,  \neg \phi \, \vert \, K \phi \, \vert  \,B (\phi\vert o)\, \vert B(\phi \vert \phi) \,\vert \, [o] \phi \,\vert \, [\phi] \phi$$
where $o_i \in O$, $a_i$'s and $c$ in $\mathbb{Q}$. The propositional connectives $\top, \neg, \wedge$ are standard. Letters $K$ and $B$ stand for knowledge and (conditional) belief operators, and  $[o]$ and $[\phi]$ capture the \emph{dynamics} of learning by an observation, $o$ and by higher order information, $\phi$ respectively, and stand for ``after observing $o$'', and ``after learning $\phi$''. Simple belief $B\phi$ is taken to be an abbreviation for $B(\phi|\top)$.

\begin{definition}[Probabilistic Plausibility Models]
A \emph{probabilistic plausibility model} over a finite set $O$ is a structure $\mathcal{M}= (M, pla, v )$ where $M \subseteq M_O$, $(M, pla)$ is a probabilistic plausibility frame and an evaluation function $v: O \to \mathcal{E}$ that assigns to each propositional variable $o$ a cylinder set $o^{j}$. \footnote{Notice that since we deal with i.i.d distributions the choice of $j$ does not matter.}
\end{definition}

\begin{definition}[\emph{Two types of update}]\label{defU} Let $\mathcal{M}=(M, pla,v)$ be a probabilistic plausibility model, let $e\in {\mathcal E}$ be a sampling event, and let $P\subseteq M$ be a higher-order ``proposition" (set of possible worlds, expressing some higher-order information about the world). The result of \emph{updating the model with sampling evidence $e$} is the model $\mathcal{M}^e=(M, pla_e,v)$. In contrast, the result of \emph{updating the model with proposition $P$} is the model $\mathcal{M}^P=(P, pla, v)$.
\end{definition}

Let $\mathcal{M}=(M, pla, v)$ be a probabilistic plausibility model. The semantics for formulas is given by inductively defining a satisfaction relation $\vDash$ between worlds and formulas. In the definition, we use the notation $\|\phi\|_{\mathcal{M}}:=\{\mu\in M \, \vert\, \mathcal{M}, \mu \vDash \phi\}$:

\begin{align*}
&\mathcal{M}, \mu \vDash \sum_{i=1}^{n} a_i w(o_{i}) \geq c&\iff   &\sum_{i=1}^{n} a_i \hat{\mu}(v(o_{i})) \geq c\\
&\mathcal{M}, \mu \vDash \phi_1 \wedge \phi_2 &\iff &\mathcal{M}, \mu \vDash \phi_1 \,\,\mathrm{ and }\,\, \mathcal{M}, \mu \vDash \phi_2 \\
&\mathcal{M}, \mu \vDash \phi_1 \wedge \phi_2 &\iff & \mathcal{M}, \mu \vDash \phi_1 \,\,\mathrm{ or }\,\, \mathcal{M}, \mu \vDash \phi_2\\
&\mathcal{M}, \mu \vDash \neg \phi &\iff& \mathcal{M}, \mu \nvDash \phi\\
&\mathcal{M}, \mu \vDash K \phi &\iff&   \mathcal{M}, \nu \vDash \phi \,\, \mathrm{ for } \,\, \mathrm{ all }\,\, \nu \in M\\
&\mathcal{M}, \mu \vDash B(\phi\, \vert \theta) &\iff& B(\|\phi\|_{\mathcal{M}}\, \vert\, \|\theta\|_{\mathcal{M}}) \,\, \mathrm{ holds } \,\, \mathrm{ in } \,\, (M, pla)\\
&\mathcal{M}, \mu \vDash B(\phi\, \vert \, o) &\iff& B(\|\phi\|_{\mathcal{M}}\, \vert\, o) \,\, \mathrm{ holds }\,\, \mathrm{ in } \,\, (M, pla)\\
&\mathcal{M}, \mu \vDash [o] \phi &\iff& \mathcal{M}^{o}, \mu \vDash \phi \\
&\mathcal{M}, \mu \vDash [\theta] \phi &\iff& \left( \mathcal{M}, \mu \vDash \theta \implies \mathcal{M}^{\theta}, \mu \vDash \phi \right)
\end{align*}

As is standard, for a model $\mathcal{M}=(M, pla, v)$, let $\norm{\phi}_{\mathcal{M}} =\{\nu \in M\, \vert \, \mathcal{M}, \nu \vDash \phi\}$ and we shall say that a formula $\phi$ is \emph{valid in $\mathcal{M}$} if and only if $\mathcal{M}, \mu \vDash \phi$ for all $\mu \in M$. Formula $\phi \in FL_{SL}$ is \emph{valid} (in the logic $L_{SL}$) if it is valid in every model $\mathcal{M}=(M, pla, v)$.
\begin{proposition}\label{prop7}
Let $\mathcal{M}$ be a probabilistic plausibility model. The set $B_{\mathcal{M}}= \{ \phi \in FL_{PU} \,\vert \, \mathcal{M} \vDash B \phi\}$ is consistent.
\end{proposition}

\begin{proof}
See appendix
\end{proof}

\begin{proposition}\label{val1} Let $o \in O$ and $\phi, \theta, \xi \in FL_{SL}$. Then the following are valid formulas in $L_{SL}$
\begin{itemize}
\item $ w(o) \geq 0$
\item $\sum_{o \in O} w(o)=1$
 \item $ K(\phi \to \theta) \to (K \phi \to K \theta)$
 \item $K \phi \to \phi $
 \item $K \phi \to KK\phi $
 \item $\neg K \phi \to K \neg K\phi $
 \item $ B(\phi \to \theta) \to (B \phi \to B \theta)$
 \item $ K \phi \to B \phi $
 \item $ B \phi \to B B \phi $
 \item $\neg B \phi \to B \neg B\phi$
\end{itemize}
\end{proposition}
\begin{proof}
Notice that at each model $\mathcal{M}$ and each world $\mu$, $w$ is interpreted as a probability mass function, namely $\mu$ itself.  The rest follow easily from the definition.
\end{proof}

The dynamic operator in our logic that correspond to learning of higher order information, $[\phi]$, is essentially an AGM type update and satisfies the corresponding axioms, that is:

\begin{proposition}\label{val2} Let $\phi, \theta, \xi \in FL_{SL}$. Then the following are valid formulas in $L_{SL}$
\begin{itemize}
 \item $ B(\phi \, \vert \, \phi)$
\item $ B (\theta\, \vert \, \phi) \rightarrow  \left(B (\xi \, \vert \, \phi \wedge \theta) \leftrightarrow B (\xi \, \vert \, \phi)\right)$
\item $ \neg  B (\neg \theta \, \vert \, \phi) \rightarrow \left(B (\xi \, \vert \, \phi \wedge \theta) \leftrightarrow B (\theta \to \xi \, \vert \, \phi)\right)$
\item If $\phi \leftrightarrow \theta$ is valid in $\mathcal{M}$ then so is $ B (\xi \, \vert \, \phi) \leftrightarrow B (\xi \, \vert \, \theta)$.
\end{itemize}
\end{proposition}
 \begin{proof}
Notice that the plausibility function induces a complete pre-order on the set of worlds. The validity of the above formulas over such frames as well as the correspondence between these formulas and the AGM axioms are given by Board in \cite{Board}.
\end{proof}

Finally, we give without proofs some validities regarding the interaction of the dynamic modalities with knowledge modality and (conditional) belief.

\begin{proposition}\label{val} Let $\phi, \theta, \xi \in FL_{SL}$. Then the following are valid formulas in $L_{SL}$
\begin{itemize}
  \item $[\phi]q \leftrightarrow (\phi \to q)$ for atomic $q$
  \item $[o]q \leftrightarrow \to q$ for atomic $q$
 \item $[\phi] \neg \theta \leftrightarrow  (\phi \to \neg [\phi] \theta)$
 \item $[o] \neg \theta \leftrightarrow  (\neg [o] \theta)$
\item $[\phi]  (\theta \wedge \xi) \leftrightarrow  ([\phi] \theta \wedge [\phi] \xi)$
\item $[o]  (\theta \wedge \xi) \leftrightarrow  ([o] \theta \wedge [o] \xi)$
 \item $[\phi] K \theta \leftrightarrow (\phi \to K[\phi] \theta)$
 \item $[o] K\phi \iff K [o]\phi$
 \item $[\phi] B (\theta\, \vert\, \xi) \iff \left( \phi \rightarrow B([\phi] \theta \, \vert \, \phi \wedge [\phi]\xi)\right)$
\item $ [o] B (\phi \, \vert o') \iff B([o]\phi \, \vert \, o, o')$
\end{itemize}
\end{proposition}

\section{Conclusion and Comparison with Other Work}

We studied  forming beliefs about unknown probabilities in the situations that are commonly described as the those of radical uncertainty.
The most widespread approach to model such situations of `radical uncertainty' is in terms of imprecise probabilities, i.e. representing the agent's knowledge as a set of probability measures. There is extensive literature on the study of imprecise probabilities \cite{Br1,Ch,Haj,Lev,Ref1,Ref2,Ref3} and on different approaches for decision making based on them \cite{BrSt,hun,Trof,Ref4,Ref5,Sei1,Sei2,wil} or to collapse the state of radical uncertainty by settling on some specific probability assignment as the most rational among all that is consistent with the agent's information. The latter giving rise to the area of investigation known as the Objective Bayesian account \cite{PR2,P1,P2,R1,W2,W3}.\\
A similar line of enquiry has been extensively pursued in economics and decision theory literature where the situation we are investigating here is referred to as Knightian uncertainty or ambiguity. This is the case when the decision maker has too little information to arrive at a unique prior. There has been different approaches in this literature to model these scenarios. These include, among others, the use of Choquet integration, for instance Heber and Strassen \cite{HubStr}, or Schmeidler \cite{sch1, sch2}, maxmin expected utility by Gilboa and Schmeidler \cite{GSC} and smooth ambiguity model by Klibanoff, Marinacci and Mukerji \cite{KMM} which employes second order probabilities or Al-Najjar's work \cite{al-naj} where he models rational agents who use frequentist models for interpreting the evidence and investigates learning in the long run. Cerreia-Vioglio et al \cite{CV} studies this problem in a formal setting similar to the one used here and axiomatizes different decision rules such as the maxmin model of Gilboa-Schmeidler and the smooth ambiguity model of Klibanoff e t al, and gives a overview of some of the different approaches in that literature. \\
These approaches employ different mechanisms for ranking probability distribution compared to what we propose in this paper. Among these it is particularly worth pointing out the difference between our setting and those ranking probability distributions by their (second order) probabilities. In contrast, in our setting, it is only the worlds with highest plausibility that play a role in specifying the set of beliefs. In particular, unlike the probabilities, the plausibilities are not cumulative in the sense that the distributions with low plausibility do not add up to form more plausible events as those with low probability would have had. This is a fundamental difference between our account and the account given in terms of second order probabilities. \\
Another approach to deal with these scenarios in the Bayesian literature come from the series of convergence results collectively referred to as ``washing out of the prior". The idea, which traces back to Savage, see \cite{sav1,sav2}, is that as long as one repeatedly updates a prior probability for an event through conditionalisation on new evidence, then in the limit one would surely converge to the true probability, independent of the initial choice of the prior.\footnote{To be more precise, if one starts with a prior probability for an event $A$, and keeps updating this probability by conditionalising on new evidence, then almost surely, the conditional probability of $A$ converges to the indicative function of $A$ (i.e. to 1 if $A$ is true, and to 0 otherwise). This form is called Levy's 0-1 law. Savage's  results use IID trials and objective probabilities and has been criticised regarding its applicability to scientific inference. There are however, a number of more powerful convergence results avoiding these assumptions, for example based on Doob's martingale convergence theorem \cite{doob}. There are also several generalisations of these results, e.g. Gaifman and Snir \cite{gs1}.} Bayesians use these results to argue that an agent's choice of a probability distribution in scenarios such as our urn example is unimportant as long as she repeatedly updates that choice (via conditionalisation) by acquiring further evidence, for example by repeated sampling from the urn.
However, it is clear that the efficiency of the agent's choice for the probability distribution, put in the context of a decision problem, depends strongly on how closely the chosen distribution tracks the actual. This choice is most relevant when the agents are facing a one-off decision problem, where their approximation of the true probability distribution at a given a point ultimately determines their actions at that point.\\
Our approach, based on forming rational qualitative beliefs \emph{about} probability (based on the agent's assessment of each distribution plausibility), does not seem prone to these objections. The agent does ``the best she can" \emph{at each moment}, given her evidence, her higher-order information and her background assumptions (captured by her plausibility map). Thus, she can solve one-off decision problems to the best of her ability. And, by updating her plausibility with new evidence, her beliefs are still guaranteed to converge to the true distribution (if given enough evidence) in essentially all conditions (-including in the cases that evade Savage-type theorems). 
We end by sketching the contours of a dynamic doxastic logic for statistical learning. Our belief operator satisfies all the axioms of standard doxastic logic, and one form of conditional belief (with propositional information) satisfies the standard AGM axioms for belief revision. But the other form of conditioning (with sampling evidence) does not satisfy these axioms, and this is in fact essential for our convergence results.

\bibliographystyle{eptcs}

\begin{thebibliography}{999}
\bibitem{al-naj}  Al-Najjar, N., ``Decision makers  as statisticians: diversity,ambiguity, and learning'', {\em Econometrica}, 77(5): 1339--1369, 2009. \doi{10.3982/ECTA7501}

\bibitem{SA} Baltag, A., Renne, B., and Smets, S., ``The Logic of Justified Belief, Explicit Knowledge and Conclusive Evidence'', {\em Annals of Pure and Applied Logic}, 165(1): 49--81, 2014.  \doi{10.1016/j.apal.2013.07.005}

\bibitem{SA2} Baltag, A., Gierasimczuk, N., and Smets, S. Stud Logica (2018).   \doi{10.1007/s11225-018-9812-x}

\bibitem{Jo} van Benthem, J., ``Dynamic Logic of Belief Revision'', {\em Journal for Applied Non-Classical Logics}, 17(2): 129--155, 2007.  \doi{10.3166/jancl.17.129-155}

\bibitem{Board} Board, O., ``Dynamic interactive epistemology", in {\em Games and Economic Behavior},49: 49--80, 2004. \doi{10.1016/j.geb.2003.10.006}
\bibitem{Br1}  Bradley, R., \& Drechsler, M., `` Types of Uncertainty", in {\em Erkenntnis}, 79: 1225--1248, 2014. \doi{10.1007/s10670-013-9518-4}
\bibitem{BrSt} Bradley, S., \& Steele, K., ``Uncertainty, Learning and the 'Problem' of Dilation``, in {\em Erkenntnis}, 79: 1287--1303, 2014. \doi{10.1007/s10670-013-9529-1}
\bibitem{CV}  Cerreia-Vioglio, S., Maccheroni, F., Marinacci, M., and Montrucchio, L., ``Ambiguity and robust statistics'', {\em Journal of Economics Theory}, 148: 974--1049, 2013. \doi{10.1016/j.jet.2012.10.003}
 \bibitem{Ch} Chandler, J., ``Subjective Probabilities Need Not Be Sharp", in {\em Erkenntnis}, 79: 1273--1286, 2014. \doi{10.1007/s10670-013-9597-2}
\bibitem{doob}Doob,J. L. ``What Is a Martingale?" in {\em American Mathematical Monthly} 78:451-462, 1971. \doi{10.2307/2317751}
\bibitem{sav1} Edwards, W., Lindman, R, and Savage, L. J.,
``Bayesian Statistical Inference for Psychological Research", in {\em Psychological Review} 70: 193-242, 1963. \doi{10.1007/978-1-4612-0919-5_34}
\bibitem{erm} Earman, J. `` Bayes or Bust: A Critical Examination of Bayesian Confirmation theory", MIT press, 1992.
\bibitem{gs1} Gaifman, H., \&  Snir, M.``Probabilities over Rich Languages", {\em Journal of Symbolic Logic} 47:495-548, 1982. \doi{10.2307/2273587}
\bibitem{GSC} Gilboa, I., and Schmeidler, D. ``Maxmin expected utility with non-unique prior'', {\em J. Math. Econ.} 18: 141--153, 1989. \doi{10.1016/0304-4068(89)90018-9}
\bibitem{Haj} Hajek, A., \& Smithson, M., ``Rationality and Indeterminate Probabilities", in {\em Synthese}, 187: 33--48, 2012. \doi{10.1007/s11229-011-0033-3}
\bibitem{HubStr} Huber, P. J., and Strassen, V., ``Minimax test and neyman-pearson lemma for capacities'', {\em The Annals of Statistics}, 1:251--263, 1973. \doi{10.1214/aos/1176342363}
\bibitem{IntroAnaly} Hunter, J. K., {\em An Introduction to Real Analysis}, https://www.math.ucdavis.edu/~hunter/m125a/intro_analysis.pdf
\bibitem{hun} Huntley, N., Hable, R., \& and Troffaes, M., ``Decision making", in {Augustin et al}, 2014.
\bibitem{KMM} Klibanoff, P., Marinacci, M., and Mukerji, S., ``A smooth model of decision making under ambiguity'', {\em Econometrica}, 73: 1849--1892, 2005. \doi{10.1111/j.1468-0262.2005.00640.x}
\bibitem{Lev} Levi, I., ``Imprecision and Indeterminacy in Probability Judgment", {\em Philosophy of Science},  52:390--409, 1985. \doi{10.1086/289257}
\bibitem{PR1} Paris, J. \& Rafiee Rad, S., ``Inference Processes for Quantified Predicate Knowledge", {\em Logic, Language, Information and Computation}, Eds.  W. Hodges \& R. de Queiroz, Springer LNAI, 5110: 249--259, 2008. \doi{10.1007/978-3-540-69937-8_22}
\bibitem{PR2} Paris, J. \& Rafiee Rad, S., ``A Note On The Least Informative Model of A Theory",
in {\em Programs, Proofs, Processes , CiE 2010}, Eds. F. Ferreira, B. Lowe,  E. Mayordomo, \& L. Mendes Gomes, Springer LNCS, 6158: 342--351, 2010. \doi{10.1007/978-3-642-13962-8_38}
\bibitem{P1} Paris, J.B. \& Vencovska, ``In defence of the maximum entropy inference process", in {\em International Journal of Approximate Reasoning},17(1): 77-103, 1997. \doi{10.1016/S0888-613X(97)00014-5}
\bibitem{P2} Paris, J. B., ``What you see is what you get", in {\em Entropy},16: 6186?6194, 2014. \doi{10.3390/e16116186}
\bibitem{R1} Rafiee Rad, S., ``Equivocation Axiom for First Order Languages", in {em Studia Logica}, 105(21), 2017.  \doi{10.1007/s11225-016-9684-x}    
\bibitem{sch1} Schmeidler, D., ``Subjective probability and expected utility without additivity'', {\em Econometrica}, 57(3):571--587, 1989. \doi{10.2307/1911053}
\bibitem{sch2} Schmeidler, D., ``Integral representation without additivity'', {\em Proceedings of the American Mathematical Society}, 97(2), 1986. \doi{10.1090/S0002-9939-1986-0835875-8}
\bibitem{Trof} Troffaesin, C. M., ``Decision making under uncertainty using imprecise probabilities", in {\em International Journal of Approximate Reasoning}, 45:17-29, 2007. \doi{10.1016/j.ijar.2006.06.001}
\bibitem{W1} Williamson, J., ``From Bayesian epistemology to inductive logic, in {\em Journal of Applied Logic}, 2, 2013. \doi{10.1016/j.jal.2013.03.006}
\bibitem{W2} Williamson, J., ``Objective Bayesian probabilistic logic, in {\em Journal of Algorithms in Cognition, Informatics and Logic}, 63:167-183, 2008. \doi{10.1016/j.jalgor.2008.07.001}
\bibitem{W3} Williamson, J., {\em In Defence of Objective Bayesianism}, Oxford University Press, 2010.
\bibitem{wally} Walley, P. ``Inferences from Multinomal Data: Learning about a bag of marbles", in {\em Journal of the Royal Statistical Society Series B}, 58:3-57, 1996. \doi{10.1111/j.2517-6161.1996.tb02065.x}
\bibitem{Ref1} Walley, P. ``Towards a unified theory of imprecise probability", in {\em International Journal of Approximate Reasoning}, 24(2): 125-148, 2000. \doi{10.1016/S0888-613X(00)00031-1}
\bibitem{Ref2} Denoeux, T., ``Modeling vague beliefs using fuzzy-valued belief structures", in {\em Fuzzy Sets and Systems}, 116(2):167-199, 2000. \doi{10.1016/S0165-0114(98)00405-9}
\bibitem{Ref3} Romeijn, J-W. \& Roy, O., ``Radical Uncertainty: Beyond Probabilistic Models of Belief", in {\em Erkenntnis}, 79(6):1221--1223, 2014. \doi{10.1007/s10670-014-9687-9}
\bibitem{Ref4} Elkin, L. \& Wheeler, G., ``Resolving Peer Disagreements Through Imprecise Probabilities", in {\em Nous}, forthcoming. \doi{10.1111/nous.12143}
\bibitem{Ref5} Mayo-Wilson, C. \& Wheeler, G. ``Scoring Imprecise Credences: A Mildly Immodest Proposal" , in {\em Philosophy and Phenomenological Research}, 93(1): 55?78, 2016. \doi{10.1111/phpr.12256}
\bibitem{Sei1} Seidenfeld, t., ``A contrast between two decision rules for use with (convex) sets of probabilities: Gamma-maximin versus E-admissibility", in {\em Synthese}, 140:69--88, 2004. \doi{ 10.1023/B:SYNT.0000029942.11359.8d}
\bibitem{Rud} Rudin, W., {\em Principles of Mathematical Analysis}, McGraw-Hill Inc, 1953.
\bibitem{sav2} Savage, L. J. { \em Foundations of Statistics}, New York: John Wiley, 1954.
\bibitem{Sei2} Seidenfeld, T., Schervish, M. J., \& Kadane, J. B., ``Coherent choice functions under uncertainty", in {\em Synthese}, 172: 157--176, 2010. \doi{10.1007/s11229-009-9470-7}
\bibitem{wil} Williams, J. \& Robert, G., ``Decision-making under indeterminacy", in {\em Philosophers' Impreint}, 14:1--34, 2014.


\end{thebibliography}

\section*{Appendix}

{\bf Proof of proposition \ref{prop6}:}

\begin{proof}
Let $\mu \in M$, then
$$(pla_{o^j})_{o'^k} (\mu)= pla_{o^j}(\mu) \hat{\mu}(o'^k)= pla(\mu) \hat{\mu}(o^j). \hat{\mu}(o'^k)= pla(\mu) \hat{\mu}(o^j \cap o'^k)$$
where the last equality follows from the independence assumption in $iid$ case.
\end{proof}

{\bf Proof of proposition \ref{prop5}:}

\begin{proof}
Let $M \subseteq M_O$ be closed. Since $pla_{e}$ is a continuous function, by Propositions \ref{comp}, \ref{compact-prop}-1 and \ref{max-attain}, there exists $\mu \in M$ such that for all $\mu' \neq \mu \in M$, $pla_{e}(\mu) \geq pla_{e}(\mu')$. Let $U_{pla_{e}}=\{ \mu \in M \, \vert \, \forall \mu' \in M \,\, pla_e(\mu) \geq pla_e(\mu')\}$. Thus $U_{pla_{e}} \neq \emptyset$. Let $\mu \in U_{pla_{e}}$ and assume $U_{pla_{e}} \subseteq P$. Then we have $\{\nu \in M \vert pla_e(\nu) \geq_e pla_e(\mu)\} = U_{pla_{e}} \subseteq P$ and thus by definition $B(P \vert e)$.
\end{proof}

To prove Theorem, we need a few well-known notions and facts:

\begin{Definition}
For $\mu \in M$, we define the set of $\mu$-normal observations as the set of infinite sequences from $O$ for which the limiting frequencies of each $o_i$ correspond to $\mu(o_i)$ and we will denote this set by $\Omega_\mu$:
$$\Omega_{\mu} =\{\omega \in \Omega \, \vert \, \forall o_i \in O \,\, \lim_{n \to \infty} \frac{\vert\{ i \leq n \, \vert \, \omega_i = o_i\} \vert }{n} = \mu(o_i)\} \setminus \{\omega \in \Omega \, \vert \, \exists i \in \mathbb{N}\,\,\, \mu(\omega_i)= 0\}.$$
\end{Definition}

\begin{proposition}\label{SLLN}
For every probability function $\mu$, $\hat{\mu}(\Omega_{\mu}) =1.$\\
Hence, if $\mu$ is the true probability distribution over $O$, then almost all observable infinite sequence from $O$ will be $\mu$-normal.
\end{proposition}

\begin{proof}
Let $\Delta =\{\omega \in \Omega \, \vert \, \exists i \in \mathbb{N} \,\,\, \mu(\omega_i)= 0\}$. Using the law of large numbers it is enough to show that $\hat{\mu}(\Delta)=0$. To see this let $\mu(o)=0$ then 
$$\hat{\mu}(\{ \omega \in \Omega \,\, \vert \,\, \exists i \in \mathbb{N} \,\,\, \omega_i =o \}) = \hat{\mu} (\bigcup_{i \in \mathbb{N}} \{ \omega \in \Omega \,\, \vert \,\, \omega_i =o \})= \sum_{i \in \mathbb{N}}\hat{\mu} (o^i) = 0.$$
The result then follows from finiteness of $O$.
\end{proof}

\begin{Lemma}\label{c1}
For $0<p_1, \ldots, p_n<1$ with $\sum p_i=1$, the function $f(\vec{x})= \Pi_{i=1}^{n} x_{i}^{ p_i}$ on domain $\vec{x} \in \{\vec{z} \in (0, 1)^n\, \vert\, \sum z_i=1\}$ has $\vec{x}=\vec{p}$ as its unique maximizer on $M_O$.
\end{Lemma}

\begin{proof} First we notice that $f(\vec{x}) \geq 0$ on $M_O=\{\vec{z} \in [0, 1]^{n} \, \vert \, \sum z_i=1\}$ and by Propositions \ref{comp} and \ref{max-attain} $f$ has a maximum on $M_O$. For any point $\vec{z} \in M_O$ with any $z_i=0$ (or $z_i=1$) $f(\vec{z}) =0$ thus $f$ reaches its maximum on $\{\vec{z} \in (0, 1)^n\, \vert\, \sum z_i=1\}$.

To show the result, we will show that $\log(f(\vec{x}))$ has $\vec{x}= \vec{p}$ as its unique maximizer on this domain. The result then follows from noticing that $f(x) \geq 0$ and the monotonicity of $\log$ function on $\mathbb{R}^{+}$. To maximise $\log(f(\vec{x}))$ subject to condition $\sum_{i} x_i=1$ we use Lagrange multiplier methods: let

$$G(\vec{x})= \log(f(\vec{x})) - \lambda (\sum_{i=1}^{n} x_i -1)= \sum_{i=1}^{n} p_i \log(x_i)- \lambda (\sum_{i=1}^{n} x_i -1).$$
Setting partial derivatives of $G$ equal to zero we get,
$$ \frac{ \partial G(\vec{x})}{\partial x_i} = \frac{p_i}{x_i} -\lambda =0$$
which gives $p_i= \lambda x_i$. Inserting this in the condition $\sum_{i} p_i =1$ we get $\lambda \sum_{i} x_i=1$ and using $\sum_{i} x_i=1$ we get $\lambda=1$ and thus $x_i=p_i$. Since $f$ has a maximum on this domain and the Lagrange multiplier method gives a necessary condition for the maximum, any point $\vec{x}$ that maximises $f$ should satisfy the condition $x_i=p_i$ and thus $\vec{p}$ is the unique maximiser for $f$.
\end{proof}

\par\noindent
{\bf Proof of Theorem \ref{learn}:}

\begin{proof}
Since $\mu(\Omega_{\mu})=1$ (by the Strong Law of Large Numbers), it is enough to show that
$$\forall \epsilon >0 \forall \omega\in \Omega_{\mu} \,\, \exists K\,\, \forall m \geq K: \,\, B (\{\nu \vert d(\mu, \nu )< \epsilon\} \, \vert \, \omega^{m}) \mbox{ holds in $M$.}$$
Let us fix some $\epsilon>0$ and some $\omega \in \Omega_{\mu}$. We need to show that, there exists $\nu \in M$ such that for all large enough $m$, for any $\xi \in M$ if $pla(\xi \, \vert \, \omega^{m}) \geq pla(\nu \, \vert \, \omega^{m})$, then $ d(\xi, \mu)< \epsilon$.
To show this, we will prove a stronger claim, namely that:
$$\exists K\,\,  \forall m \geq K \,\, \forall \nu\in M_O \left( d(\nu, \mu) \geq \epsilon \,\Rightarrow \, pla(\mu \, \vert \, \omega^{m}) > pla(\nu \, \vert \, \omega^{m})\right).$$
(Note that the desired conclusion follows immediately from this claim: since we can then take $\mu$ itself to be the desired $\nu \in M$. Then by the above claim, no measures $\xi$ in $M_O$ with $d(\mu, \xi) \geq \epsilon$ satisfies $pla(\xi \, \vert \, \omega^{m}) \geq pla(\mu\, \vert \, \omega^{m})$ and thus all measures, $\nu$ that satisfy this inequality have to satisfy $d(\mu, \nu) < \epsilon$.)
By definition, for all $\nu\in M_O$ we have $pla(\nu \,\vert \, \omega^m)= pla(\nu)\cdot \hat{\nu}(\omega^m)$.
By independence, we obtain that $pla(\nu\,\vert \, \omega^m)= pla(\nu)  \cdot \Pi_{i=1}^{n} \nu(o_i)^{m_i} = pla(\nu)\cdot \Pi_{i=1}^n \nu_i^{m\cdot \alpha_{i,m}},$
were we have put  $\nu_i:=\nu(o_i)$ and $\alpha_{i,m}=\frac{m_i}{m}$, for all $1\leq i\leq n$ and all $m\in N$. Note that, since $\omega\in \Omega_{\mu}$, we have that
$\lim_{m \to \infty} \alpha_{i,m}=p_i, \,\, \mbox{ for all $1\leq i\leq n$},$
where we had put $p_i:=\mu(o_i)$, for $1\leq i\leq n$. In particular, for $\nu=\mu$ (so $\nu_i=\mu(o_i)=p_i$), we obtain that
$pla(\mu\,\vert \, \omega^m)= pla(\mu)\cdot \Pi_{i=1}^n p_i^{m\cdot \alpha_{i,m}}$. Notice also that by definition of $\Omega_{\mu}$ if $p_i= \mu(o_i)=0$ then $\alpha_{i,m}=0$. Hence from the assumption that $pla(\mu) \neq 0$ we have $pla(\mu\,\vert \, \omega^m)= pla(\mu)\cdot \Pi_{i=1}^n p_i^{m\cdot \alpha_{i,m}} > 0$.

To prove the desired conclusion, it is enough (by the above representations of $pla(\nu\,\vert \, \omega^m)$ and $pla(\mu\,\vert \, \omega^m)$) to show that, for all big enough $m$ and all $\nu \in M_O \setminus B_{\epsilon}(\mu)$, we have
\begin{equation}\label{eqt1}
pla(\nu)\cdot
\Pi_{i=1}^n \nu_i^{m\cdot \alpha_{i,m}} < pla(\mu)\cdot
\Pi_{i=1}^n p_i^{m\cdot \alpha_{i,m}}
\end{equation}

We consider this in two cases. For the first case, assume that $pla(\nu)=0$, then the left hand side of (\ref{eqt1}) is $0$ and the inequality holds. For the second case, let $pla(\nu) > 0$ and let $A=\{1 \leq i \leq n \, \vert p_i \neq 0\}$. To show (\ref{eqt1}) it is enough to show that

\begin{equation}\label{eqt111}
pla(\nu)\cdot
\Pi_{i \in A} \nu_i^{m\cdot \alpha_{i,m}} < pla(\mu)\cdot
\Pi_{i \in A} p_i^{m\cdot \alpha_{i,m}}
\end{equation}

Since $\lim_{m \to \infty} \alpha_{i,m}=p_i$, there must exist some $N_1$ such that $\frac{p_i}{2}\leq \alpha_{i,m}\leq
2\cdot p_i$ for all $m\geq N_1$ and all $i \in A$. Let $\Delta=\{\nu\in M_O \, \vert \, \nu(o_i)=0 \mbox{ for some } i \in A \}$, and similarly for
any $\delta>0$, put $\Delta_{\delta}=\{\nu\in M_O \, \vert \, \nu(o_i)<\delta \mbox{ for some } i \in A\}$, and so
$\overline{\Delta_{\delta}}=\{\nu\in M_O \, \vert \, \nu(o_i)\leq\delta \mbox{ for some } i \in A\}$ is its closure. Choose some $\delta>0$ small enough such that we have
$\Pi_{i \in A} \nu_i^{2\cdot p_i} < \Pi_{i \in A} p_i^{\frac{p_i}{2}}$ for all $\nu \in \overline{\Delta_{\delta}}$ (-this is possible, since $\Pi_{ii \in A} \nu_i^{2\cdot p_i}=0< \Pi_{i \in A} p_i^{\frac{p_i}{2}}$ for all $\nu\in \Delta$, so the continuity of $\Pi_{i \in A} \nu_i^{2\cdot p_i}$ gives us the existence of $\delta$). Hence, we have
$$0\leq \frac{\Pi_{i \in A} \nu_i^{2\cdot p_i}}{\Pi_{i \in A} p_i^{\frac{p_i}{2}}}<1 \,\, \mbox{ for all $\nu \in \overline{\Delta_{\delta}}$}.$$
The set $\overline{\Delta_{\delta}}$ is closed, hence the continuous functions $pla(\nu)$ and $\frac{\Pi_{i \in A} \nu_i^{2\cdot p_i}}{\Pi_{i \in A} p_i^{\frac{p_i}{2}}}$ attain their supremum (maximum) on $\overline{\Delta_{\delta}}$. Let $K<\infty$ be the maximum of $pla(\nu)$, and $Q<1$ be the maximum of $\frac{\Pi_{i \in A} \nu_i^{2\cdot p_i}}{\Pi_{i \in A} p_i^{\frac{p_i}{2}}}$ on this set (-the fact that $Q<1$ follows from the inequality above).
Then there exists some $N_2>N_1$, s.t. we have $Q^m< \frac{pla(\mu)}{K}$ for all $m>N_2$. Hence, for all $\nu\in \Delta_{\delta}$, we have:

$pla(\nu)\cdot \Pi_{i \in A} \nu_i^{m\cdot \alpha_{i,m}}\leq K\cdot \Pi_{i \in A} \nu_i^{m\cdot 2\cdot p_i}
\leq K\cdot (Q\cdot \Pi_{i \in A} p_i^{\frac{p_i}{2}})^m = K\cdot Q^m \cdot  \Pi_{i \in A} p_i^{m\cdot \frac{p_i}{2}}
< K\cdot \frac{pla(\mu)}{K}\cdot \Pi_{i \in A} p_i^{m\cdot \alpha_{i,m}}= pla(\mu)\cdot \Pi_{i \in A} p_i^{m\cdot \alpha_{i,m}}$

So we proved that the inequality (\ref{eqt1}) holds on $\Delta_{\delta}$. It thus remains only to prove it for all
$\nu \in M':=M_O- (B_{\epsilon}(\mu)\cup \Delta_{\delta})$, where $B_{\epsilon}(\mu)=\{\nu \in M_O \, \vert \, d(\mu, \nu)< \epsilon\}$. For this, note that $M':=M_O- (B_{\epsilon}(\mu)\cup \Delta_{\delta})$ is closed and that $\nu_i\not=0$ over this set for all $i \in A$ and for all $i \notin A$, $\alpha_{i,m} =0$. Hence using the assumption that $pla(\nu) \neq 0$, (\ref{eqt1}) is equivalent over this set with:
\begin{equation}\label{eqt3}
\left(\frac{pla(\mu)}{pla(\nu)}\right)  \cdot \left( \frac{\Pi_{i=1}^{n} p_i^{m\cdot \alpha_{i,m}}}{ \Pi_{i=1}^{n} \nu_i^{m \cdot \alpha_{i,m}}} \right) >1.
\end{equation}
Applying logarithm (and using its monotonicity, and its other properties), this in turn is equivalent to
\begin{equation}\label{eq666}
\log(pla(\mu))-\log (pla(\nu))
+  \sum_{i=1}^{n} m\cdot \alpha_{i,m} \cdot (\log p_i - \log \nu_i) > 0.\end{equation}
So we see that it is enough to show that, for all large $m$ and for $\nu \in M'$, we have
\begin{equation}\label{EqMain}
m> \frac{log(pla(\nu))-log(pla(\mu))}{\sum_{i=1}^{n} \alpha_{i,m} \cdot (\log p_i - \log \nu_i)}
\end{equation}
Recall that $\alpha_{i,m}\geq \frac{p_i}{2}$ for all $m> N_2> N_1$ and all $1\leq i\leq n$. Thus, to prove (\ref{EqMain}), it is enough to show that, for large $m$ and for all $\nu \in M'$, we have
\begin{equation}\label{EqNew}
m> \frac{f(\nu)}{g(\nu)},
\end{equation}
where we introduced the auxiliary continuous functions $f, g: M'\to R$, defined by putting $f(\nu)= 2\cdot(\log(pla(\nu))-\log (pla(\mu)))$ and
$g(\nu)=  \sum_{i=1}^{n} p_i \cdot (\log p_i - \log \nu_i)$ for all $\nu\in M_O$.

To show (\ref{EqNew}), note first that
$$g(\nu)= \sum_{i=1}^n p_i\cdot (\log p_i - \log \nu_i)= log \left(\frac{\Pi_{i=1}^n p_i^{p_i}}{\Pi_{i=1}^n \nu_i^{p_i}}\right)> log 1=0$$
(where at the end we used the fact, proved in Lemma \ref{c1}, that the measure $\mu$, with values $\mu(o_i)=p_i$, is the unique maximizer of the function $\Pi_{i=1}^n \nu_i^{p_i}$ on $M_O$). Since $g$ is continuous and $M'$ is closed, $g$ is bounded and attains its infimum $A=min_{M'}(g)$ on $M'$. But since $g$ is non-zero on $M'$, this minimum cannot be zero: $A=min_{M'}(g)\not=0$. Similarly, since $f$ is continuous and $M'$ is closed, $g$ is bounded and attains its supremum
$B=max_{M'}(f)<\infty$ (which thus has to be finite). Take now some $N \geq max(N_2, \frac{B}{A})$. For all $m>  N$, we have
$$m > \frac{B}{A}\geq \frac{f(\nu)}{g(\nu)}$$
for all $\nu\in M'$, as desired.
\end{proof}

{\bf Proof of proposition \ref{prop7}:}

\begin{proof}
Take a probabilistic plausibility model $\mathcal{M}=(M, pla, v)$. Let $\phi \in B_{\mathcal{M}}$. We show that for any $\xi \in M$ there is some member of $M$ that is at least as plausible as $\xi$ but does not belong to $\neg \phi$ and thus by definition $\neg \phi \not \in B_{\mathcal{M}}$.

Since $\phi \in B_{\mathcal{M}}$, by definition there exits $\mu \in M$ such that for all $\nu \in M$ with $pla(\nu) \geq pla(\mu)$, $\nu \in \norm{\phi}$. Then if $pla(\xi) \geq pla(\mu)$, then $\xi \in \norm{\phi}$ and thus $\xi \notin M \setminus \norm{\phi}= \norm{\neg \phi}$. Thus there exists some elements of $M$, namely, $\xi$ itself that is at least as plausible of $\xi$ but does not belong to $\norm{\neg \phi}$. If $pla(\zeta) < pla(\mu)$ and since $\mu \in \norm{\phi}$, $\mu \notin M \setminus \norm{\phi}=\norm{\neg \phi}$. Then again there is some member of $M$, namely $\mu$ that is more plausible than $\xi$ but does not belong to $\norm{\neg \phi}$.
\end{proof}

\end{document}